\documentclass[conference]{IEEEtran}
\usepackage{graphicx}
\usepackage{amsmath,amssymb, amsthm} 
\theoremstyle{definition}
\newtheorem{definition}{Definition}
\theoremstyle{proposition}
\newtheorem{proposition}{Proposition}

\usepackage{lineno}
\usepackage{color}

\usepackage[utf8]{inputenc} 
\usepackage[T1]{fontenc}    
\usepackage{hyperref}       
\usepackage{url}            
\usepackage{booktabs}       
\usepackage{amsfonts}       
\usepackage{nicefrac}       
\usepackage{microtype}      

\usepackage[shortlabels]{enumitem}
\usepackage{xcolor}
\usepackage{mathtools} 
\usepackage{algorithm}
\usepackage{algpseudocode}
\usepackage{tikz}
\usepackage{bbm}
\usepackage{notetaking}
\usepackage{cleveref} 
\usepackage[caption=false]{subfig}

\definecolor{fg}{rgb}{.13, .55, .13}

\begin{document}

\title{Sparse Super--Regular Networks} 

\author{\IEEEauthorblockN{Andrew W.E. McDonald}
\IEEEauthorblockA{\textit{Department of Computer Science} \\
\textit{Drexel University}\\
Philadelphia, PA \\
awm32@drexel.edu}
\and
\IEEEauthorblockN{Ali Shokoufandeh}
\IEEEauthorblockA{\textit{Department of Computer Science} \\
\textit{Drexel University}\\
Philadelphia, PA \\
as79@drexel.edu}
}
\maketitle

\begin{abstract}
    It has been argued by Thom and Palm \cite{thom2016} that sparsely--connected neural networks (SCNs) show improved performance over fully--connected networks (FCNs).
    Super--regular networks (SRNs) are neural networks composed of a set of stacked sparse layers of ($\epsilon,\delta$)--super--regular pairs, and randomly permuted node order.
    Using the Blow--up Lemma, we prove that as a result of the individual super--regularity of each pair of layers, SRNs guarantee a number of properties that make them suitable replacements for FCNs for many tasks.
    These guarantees include edge uniformity across all large--enough subsets, minimum node in-- and out--degree, input--output sensitivity, and the ability to embed pre--trained constructs.
    Indeed, SRNs have the capacity to act like FCNs, and eliminate the need for costly regularization schemes like Dropout.
    We show that SRNs perform similarly to X--Nets via readily reproducible experiments, and offer far greater guarantees and control over network structure.
\end{abstract}

\section{Introduction}

Deep neural networks (DNNs) are widely applied in a broad range of fields including healthcare \cite{esteva2019guide}, environmental sciences \cite{mayr2016deeptox}, and computer vision (CV) and machine learning tasks including object detection, classification, segmentation, and pattern recognition \cite{lecun2015deep}.
While DNNs have widespread general applicability to a large swath of problems, large fully connected DNNs are prone to over-training and are computationally expensive \cite{srivastava14a}.

Thom and Palm argue that sparsely--connected neural networks (SCNs) show improved performance over FCNs \cite{thom2016}, and as discussed by \cite{prabhu2017deep}, FCNs require more space and time resources than required by SCNs to produce only slightly more accurate results, if at all. 
However, the intractability of neural network edge assignment has pushed deep learning toward using stochastic methods for learning acceptable sparse edge assignments. 
Dropout is one example of a randomized and costly regularization method that has been popularized to combat the overfitting problem introduced by using more edges than necessary.

Super--regular networks (SRNs) offer a viable construction for sparsely--connected neural networks with \textit{near} uniform density across all large--enough subsets of nodes. This is partly due to the guarantee that the sparsity of such subsets is bounded. Ideally, SRNs will satisfy the Blow--up Lemma \cite{komlos1997blow} property, which states that bipartite graphs satisfying the $(\epsilon,\delta)$--super--regularity conditions behave like complete bipartite graphs subject to practically realizable constraints. 
As far as the deep learning community is concerned, this means that SRNs \textit{have the capacity to approximate fully--connected networks}, despite employing significantly fewer edges. In addition, SRNs' pseudo-deterministic edge generation provides greater control over network architecture, while the randomized node permutation ensures proper mixing \textit{while retaining the super--regular properties imposed by deterministic edge assignment}.

An SRN is a set of stacked bipartite graphs, each of which is an $(\epsilon,\delta)$--balanced matrix of variable size. Each ($\epsilon,\delta$)--balanced matrix stands for an ($\epsilon,\delta$)--super--regular pair, which itself is a pairwise disjoint bipartite graph.  In the case of an ($\epsilon,\delta$)--balanced matrix, the rows and columns serve as the left and right parts of the bipartite graph described by the super--regular pair. Consecutive pairs of left and right parts of pairwise disjoint bipartite graphs together form a neural network composed of sequential but independent pseudo super--regular pairs.
This architecture results in a controllably sparse neural network with the potential to act as an FCN.


\textbf{Our contributions are:} \textit{1)} we introduce the notion of $\epsilon$-- and ($\epsilon,\delta$)--balanced matrices; \textit{2)} we establish a relationship  between $\epsilon$--balanced matrices and $\epsilon$--regular pairs, and between ($\epsilon,\delta$)--balanced matrices and ($\epsilon,\delta$)--super--regular pairs; \textit{3)} we present a deterministic way to construct SRNs, subject to randomly permuted node ordering; \textit{4)} we show that SRNs produce comparable results to a family of related sparse networks known as X--Nets, while offering greater guarantees about, and much more control over network architecture.

The remainder of the paper is structured as follows. 
\Cref{sec:relatedwork} identifies recent related work to this one.
\Cref{sec:srntheory} gives an overview of $\epsilon$--regularity and $(\epsilon,\delta)$--super--regularity. Then it defines $\epsilon$--balanced and $(\epsilon,\delta)$--balanced matrices, and proves their respective equivalence with $\epsilon$--regular and $(\epsilon,\delta)$--super--regular pairs. Finally, it defines super--regular networks, and illustrates their advantages over X--Nets \cite{prabhu2017deep}.
\Cref{sec:detcon} describes a deterministic construction of SRNs, and \cref{sec:experiments} presents an empirical evaluation of SRNs as compared to X--Nets and FCNs.
\Cref{sec:discussion} is a discussion of the performance, merits, and current shortcomings of SRNs, and \cref{sec:conclusions} gives an overview of the material presented and suggests directions for future work. 

\section{Related Work}
\label{sec:relatedwork}

Srivastava et al. introduced Gaussian dropout, a regularization technique to minimize overfitting via co--adaptation. It has proven useful in increasing DNN performance in a variety of fields including computational biology, computational vision, and speech recognition \cite{srivastava14a}.
While Gaussian dropout forces the network to learn a sparse representation \cite{srivastava14a}, Molchanov et al. showed that sparse variational dropout creates a sparse network by zeroing out frequently dropped weights \cite{molchanov2017variational}.
Via the MNIST dataset, Thom and Palm showed that sparse connectivity has the potential to boost classification performance \cite{thom2016}.

Guo et al. show that sparse, nonlinear DNNs are consistently more adversarially robust than their FCN counterparts, but that ``over--pruned'' networks are more susceptible to adversarial attacks like DeepFool \cite{guoSparseAdversarial,moosavi2016deepfool}.
Wen et al. learn a sparse, more efficient network structure, by removing less important filters and channels as part of their optimization \cite{wen2016lss}. 
Similarly, Tartaglione et al. use a regularization term to gradually prune away parameters that have little impact on the output, resulting in very sparse but accurate networks \cite{tartaglione2018}.
Zhu et al. used a ``decorrelation'' regularization term along with group LASSO regularization to learn a sparse CNN with decorrelated convolution filters \cite{zhu2018improving}. 
Sun et al. used iterative, per--layer training to create sparse CNNs for facial recognition \cite{sun2016cvpr}.

Prabhu et al. tie extremal graph theory into deep learning, in their presentation of X--Nets \cite{prabhu2017deep}, which showed comparable performance to FCNs.
X--Nets are sparse neural networks constructed from a set of randomly generated, stacked bipartite expander graphs. 
While input--output sensitivity is guaranteed due to the random edge assignment, this approach cannot guarantee a minimum node degree, which creates the potential for isolated subgraphs.
Komlos et al. discuss the importance of Szemer\'edi's Regularity Lemma and associated Blow--up Lemma as it applies to embedding bounded degree subgraphs \cite{komlos2000regularity}, however do not present a deterministic construction of ($\epsilon,\delta$)--super--regular pairs. Kalantari et al. analyze the time complexity of balancing a matrix \cite{kalantari1997}, but do not extend their analysis to matrices balanced within some $\epsilon$ parameter.

This work ties these concepts together by presenting a deterministic and tunable construction of sparse neural networks---in the form of SRNs---via a pseudo--deterministic construction of ($\epsilon,\delta$)--super--regular pairs, and by necessity introduces the notion of $\epsilon$-- and ($\epsilon,\delta$)-- balanced matrices. As a result, all expander networks (X--Nets) that are also SRNs have properties that X--Nets alone cannot guarantee.

\section{Super--Regular Networks: Theory}
\label{sec:srntheory}

In this section, we first describe some helpful notation, including $\epsilon$--regularity and an $\epsilon$--balanced matrix, followed by the definition of  ($\epsilon,\delta$)--super--regularity and an ($\epsilon,\delta$)--balanced matrix. 
After extending the definition of an ($\epsilon,\delta$)--balanced matrix to non--square matrices, we introduce X--Nets and briefly contextualize them with respect to super--regular pairs.
Finally, we discuss the advantages of SRNs over X--Nets.

\noindent The edge density between two vertex sets  $A$ and $B$, as presented by Komlos et al.~\cite{komlos1997blow} is
\begin{equation}
d(A, B) = \frac{e(A,B)}{|A||B|},
\label{eq:density}
\end{equation}
where $e(A,B)$ represents the number of edges between sets $A$ and $B$. Throughout out this manuscript, $G$ always refers to a bipartite graph, while $A$ and $B$ denote the left and right pairwise disjoint subsets (``parts'') of $G$, respectively.
\subsection{Regularity and Balanced Matrices}
As discussed in \cite{komlos1997blow}, given a bipartite graph $G$ on vertex set $A\cup B$, the pair $(A, B)$ is $\epsilon$--regular if and only if for  any subset pair $(X,Y)$, with $X \subset A$ and $Y \subset B$, that satisfy $|X| > \epsilon|A|$ and $|Y| > \epsilon|B|$ we have 
\begin{equation}
    | d(X, Y) - d(A, B)| < \epsilon.
\label{eq:regularityDensity}
\end{equation}
This means that for subsets $X$ and $Y$ larger than $\epsilon m$, the difference in edge density between subsets $X$ and $Y$ and the entire graph will be very small (less than $\epsilon$). 
The practical implication is that an $\epsilon$--regular graph, $G$, will be nearly uniform, and that all (large--enough) subsets of $G$ behave almost exactly like $G$ as a whole.

For computational purposes, we present $\epsilon$--regular pairs as a matrix.
An $
n \times n$ matrix with non-negative values is balanced if the sum of values in row $i$ and column $i$ are equivalent \cite{kalantari1997}.
The definition of an $\epsilon$--balanced matrix follows naturally.

\begin{definition}
\label{epsilonBalancedDefinition}
    Let $Q$ be an $m \times m$ matrix, where $M = \{1 \dots m\}$. Then the set of rows of $Q$ is denoted by $A = Q_M$,  and the set of columns is denoted $B = Q^M$. The density of $Q$, $d(Q)$ is given by \cref{eq:dQ}. 
    We obtain $\epsilon\prime$ by substituting $d(Q)$ for $d(A,B)$ in \cref{eq:regularityDensity}, as shown in \cref{eq:eBalDensity}.
    $Q$ is $\epsilon$--balanced if for every pair $(X,Y)$ where $X \subset A$, and $Y \subset B$, with $|X| > \epsilon|A|$ and $|Y| > \epsilon|B|$, that satisfy  \cref{eq:regularityDensity}, $\epsilon\prime \le \epsilon$ from \cref{eq:eBalDensity}.
    
\begin{equation}
    d(A, B) = d(Q) = \frac{\mathbbm{1}^T Q \mathbbm{1}}{|A||B|},
\label{eq:dQ}
\end{equation}

\begin{equation}
\epsilon\prime =  \left | d(X,Y) - d(Q) \right |.
\label{eq:eBalDensity}
\end{equation}
\end{definition}

\begin{proposition}
    If an $m \times m$ matrix, $Q$, is $\epsilon$--balanced, then a bipartite $\epsilon$--regular pair, $G = (A,B)$, may be constructed from it by creating a $2m \times 2m$ adjacency matrix, $D$, such that

\begin{equation}
    D =
\left[
\begin{array}{c|c}
\emptyset & Q \\
\hline
Q\prime & \emptyset
\end{array}
\right]
\label{eq:dAdj}
\end{equation}
where $M = \{1\dots m\}$, $A = Q_M$ (rows) and $B = Q^M$ (columns).
\label{eBalProp}
\end{proposition}
\begin{proof}

By definition, $Q_N$ and  $Q^N$ are disjoint subsets of $G$. Select subsets $X \subset A = Q_N$ and $Y \subset B = Q^N$.
Then $\forall (X, Y)$ satisfying $|X| > \epsilon|A|$ and $|Y| > \epsilon|B|$, \cref{eq:regularityDensity} must also be satisfied.
If \cref{eq:regularityDensity} were not satisfied, $\epsilon\prime$ from \cref{eq:eBalDensity} would be greater than $\epsilon$, and $Q$ could not be $\epsilon$--balanced.
Since  $|X| > \epsilon|A|$, $|Y| > \epsilon|B|$ and \cref{eq:regularityDensity} are all satisfied, $(A,B)$ must be an $\epsilon$--regular pair.

\end{proof}


\noindent Again, from \cite{komlos1997blow}, for a graph $G$ to be ($\epsilon,\delta$)--super--regular, in addition to satisfying the conditions for $\epsilon$--regularity, all $(X,Y)$ pairs satisfying $|X| > \epsilon|A|$ and $|Y| > \epsilon|B|$, must also satisfy

\begin{equation}
e(X,Y) > \delta|X||Y|,
\label{deltaCondOne}
\end{equation}

and 

\begin{equation}
    \forall \ a \in A,\ b \in B\text,\ deg(a) > \delta|B|\ \&\ deg(b) > \delta|A|.
\label{deltaCondTwo}
\end{equation}
The condition described by \cref{deltaCondOne} ensures the edge density between $X$ and $Y$ is greater than $\delta$, while the condition described by \cref{deltaCondTwo} requires all vertices in $G$ to have a minimum degree, bound by $\delta$.
When $|A| = |B| = \frac{|G|}{2}$, $\delta$ puts the same minimum bound on all of $G$'s vertices. Again, for computational purposes, we present a related property for $(\epsilon,\delta)$--super--regular pairs. The definition of an $(\epsilon,\delta)$--balanced matrix follows from an $\epsilon$--balanced matrix.

\begin{definition}
	Let $Q$ be an $m \times m$ $\epsilon$--balanced matrix, where $M = \{1 \dots m\}$. Then the set of rows of $Q$ is denoted by $A = Q_M$,  and the set of columns is denoted $B = Q^M$. Then $Q$ is $(\epsilon,\delta)$--balanced if for every $(X,Y)$ pair such that $X \subset A$ and $Y \subset B$, with $|X| > \epsilon|A|$ and $|Y| > \epsilon|B|$, that satisfy \cref{eq:regularityDensity}, \cref{deltaCondOne,deltaCondTwo} are also satisfied.
\label{epsilonDeltaBalancedDefinition}
\end{definition}


\begin{proposition}
    If an $m \times m$ matrix, $Q$, is ($\epsilon,\delta$)--balanced, then a bipartite ($\epsilon$,$\delta$)--super--regular pair, $G = (A,B)$, may be constructed from it by creating a $2m \times 2m$ adjacency matrix, $D$, as per \cref{eq:dAdj}.

\label{edBalProp}
\end{proposition}

\begin{proof}

Again, by definition, $Q_M$ and  $Q^M$ are disjoint subsets of $G$. Select subsets $X \subset A = Q_M$ and $Y \subset B = Q^M$.
Then $\forall (X, Y)$ satisfying $|X| > \epsilon|A|$ and $|Y| > \epsilon|B|$, \cref{eq:regularityDensity,deltaCondOne,deltaCondTwo} must also be satisfied. 
If this were not the case, $Q$ could not be ($\epsilon,\delta)$--balanced.
Since \cref{eq:regularityDensity,deltaCondOne,deltaCondTwo} are all satisfied, $(A,B)$ must be an $(\epsilon,\delta)$--super--regular pair.

\end{proof}

\subsection{Non--Square ($\epsilon,\delta$)--Balanced Matrices}

Without loss of generality we show that the result of column--wise concatenation, denoted by \textit{ccat}($Q_1, Q_2$), of two ($\epsilon,\delta$)--balanced matrices, is an ($\epsilon,\delta$)--balanced matrix, provided that $Q_1$ and $Q_2$ have at least one dimension in common, are concatenated along that dimension, and have the same density. 
    \begin{proposition}
        The result of column--wise concatenation of two ($\epsilon,\delta$)--balanced matrices  $U$ and $v$, with dimensions ($n \times m$) and $V$ ($n \times r$) respectively, is an $[n \times (m+r)]$ ($\epsilon,\delta$)--balanced matrix if and only if $d(U) = d(V)$, and $deg(U) = deg(V)$.
    \end{proposition}

    \begin{proof}
    Let $T = \textit{ccat}(U, V)$.
	Using \cref{eq:dQ}, let $d(A_u, B_u) = d(U)$, $d(A_v, B_v) = D(V)$, and $d(A_t, B_t) = D(T)$. Further, let $X_u \subset A_u$, $Y_u \subset B_u$, $X_v \subset A_v$, $Y_v \subset B_v$,  for all subsets that independently satisfy $|X| > \epsilon|A|$ and $|Y| > \epsilon|B|$. Since
 \begin{equation}
  d(T) = \frac{\mathbbm{1}^T  T\mathbbm{1}}{n(m+r)} = \frac{1}{2}\left(d(U) + d(V)\right),
\end{equation}
and 
\begin{align}
  \epsilon_u & > | d(X_u, Y_u) - d(U)|,\\
 \epsilon_v & > | d(X_v, Y_v) - d(V) |, 
\end{align}
we have
\begin{equation}
  \epsilon_t \ge \max(\epsilon_u, \epsilon_v) > \max(|d(X_u, Y_u) - d(T)|, | d(X_v, Y_v) - d(T) |) .
\end{equation}
This shows that the matrix resulting from the column--wise concatenation of two ($\epsilon,\delta$)--balanced matrices will only satisfy  \cref{eq:regularityDensity} if they retain the same $\epsilon$, which requires $d(T) = d(U) = d(V)$.
With respect to \cref{deltaCondOne,deltaCondTwo}, if $ \min(deg(A_u)) \ne \min(deg(A_v)))$ or $ \min(deg(B_u)) \ne \min(deg(B_v)))$, then $\delta_t$ is chosen by \cref{eq:deltaT},

\begin{align}
\delta_{t,B} |B_t| & = \delta_t |B_u + B_v| = \min(\delta_u, \delta_v) |B_u + B_v|, \\
\delta_{t,A} |A_t| & = \min(\delta_u, \delta_v) |A_t|, \\
\delta_t  & = \min(\delta_{t,B}, \delta_{t,A}).
\label{eq:deltaT}
\end{align}
This new $\delta_t$ satisfies \cref{deltaCondOne,deltaCondTwo} for $T$, and $\delta_t = \delta_u = \delta_v$ if $U$ and $V$ have the same \textit{ratio} of $deg(A):deg(B)$ (in--degree to out--degree).

        

    \end{proof}

\subsection{X--Nets and ($\epsilon,\delta$)--Super--Regularity of Expander Graphs}

X--Nets are sets of stacked, randomly created bipartite expander graphs. 
A bipartite graph, $G$ such that part $a \in A$ has $H$ neighbors in $B$, and spectral gap $\gamma \le 1 - \frac{|\lambda_2|}{H}$ is said to be an expander graph \cite{prabhu2017deep}.

Using some notation from \cite{komlos1997blow}, assume each layer of a deep expander network (X--Net) is the product of ``blowing--up'' a graph with the following structure, $G =  L_1 \rightarrow L_2 \rightarrow \cdots L_{r}$, where each vertex set, $L_i$ represents a layer of the network (from 1 to $r$).
Assume $|L_i| = n$ $\forall i \in 1..r$, and $V = \cup L_i$.
Further, assume that edges are uniformly randomly assigned between all successive pairs of sets (layers) such that each pair fulfills the below requirements: 

    \begin{enumerate}[$R_1$:]
        \item $|S| > \epsilon|L_1|$ and $|N(S)| > \epsilon|L_2|$
        \item $|d(S, N(S)) - d(L_1, L_2)| < \epsilon$
        \item
            $deg(l_1) \ge \delta|L_2|$ $\forall l_1 \in L_1$, and
            $deg(l_2) \ge \delta|L_1|$ $\forall l_2 \in L_2$
    \end{enumerate}

 The layers of \textit{some} X--Nets may be modeled via $(\epsilon,\delta)$--super--regular pairs (precise conditions will be discussed later), and the ones that cannot are \textit{too} sparse, and $R_2$ cannot be satisfied. If $R_1$ and $R_2$ are both satisfied, the pair is $\epsilon$--regular. If $R_3$ is also fulfilled, then the pair is $(\epsilon,\delta)$--super--regular. If all successive $L_i$ pairs satisfy all three above requirements, there exists an embedding of the X--Net into some super--regular network, $P$. This is a result of Theorem 1 from \cite{komlos1997blow}, which states that if a graph may be embedded into the fully--connected ``blown--up'' structure of some graph, F, if can also be embedded into a sparse version of F where the edges have been replaced such that each pair of vertex sets satisfies the conditions for super--regularity.

    In order to use $(\epsilon, \delta)$--nets (SRNs) instead of X--Nets, $(\epsilon, \delta)$ ranges must be found that bound the number of layers required to guarantee every output is sensitive to every input.

    This can be done by choosing parameters that ensure $P$ meets the criteria for an expander network.
   $R_1$ states that $\epsilon < \frac{|S|}{n}$.  $R_2$ states the density of any subset (greater than some size), must not differ from the density of the two layers in question by more than $\epsilon$.  The minimum density between $S$ and $N(S)$, $d(S, N(S))$ is
 \begin{equation}
    d = \frac{D_{min}|S|} {|S||N(S)|},
\end{equation}
where $D_{min}$ is the minimum degree of any vertex in $L_1$ or $L_2$ (multiplying this by the size of $S$ gives the minimum number of edges between $S$ and $N(S)$). The minimum density between $L_1$ and $L_2$ is:
$$ d = \frac{D_{min}(n)}{n^2} .$$
Then, $R_2$ becomes:
\begin{equation}
D_{min}\left|\frac{1} {|N(S)|} - \frac{1} {n}\right| < \epsilon < \frac{|S|} {n} < \frac{1} {2}
 \label{eq:simplified_r2}
\end{equation}
If we set $|S| = 1$, the condition remains satisfied for a single starting vertex (input). 


\subsection{Advantages of SRNs over X-Nets}
\label{ssec:advantages}
    First, SRNs put lower bounds on X--Net sparsity.
    \Cref{eq:simplified_r2} describes the relationship between the minimum size of $S$ (in $L_1$) and its the neighborhood in $L_2$, $N(S)$, the minimum degree of the graph, size of each partition, and $\epsilon$.
    Effectively, it says that when expanders become too sparse, they no longer satisfy the conditions of super-regularity.
    Along with $R_3$, this shows that \textit{all} bipartite expander graphs greater than a given density may be expressed as $(\epsilon,\delta)$-super-regular pairs. 

    Because we must be able to satisfy this condition with a $|S| = 1$, and $D_{min} = D$ for a $D-regular$ expander,

	\begin{equation}
	\left|\frac{D} {|N(S)|} - \frac{D} {n}\right| < \frac{1} {n},
	\end{equation}
which is equivalent to

\begin{equation} 
\left| d(S,N(S)) - \frac{D} {n} \right| \le \epsilon \le \frac{(1 - \gamma)\sqrt{|S||N(S)|}} {|S||N(S)|}.
\label{eq:expander_density_guarantee}
\end{equation}

This means that expander graphs force $\epsilon \le \frac{(1 - \gamma)\sqrt{|S||N(S)|}} {|S||N(S)|}$, while $(\epsilon,\delta)$-pairs further restrict the density differential between vertex set pairs $S$, and $N(S)$ to $\epsilon < \frac{1} {n}$, while maintaining the expansion property. 
Second, SRNs guarantee minimum degree, and therefore connectivity.
$R_3$ above enables us to guarantee that \textit{no vertex} will have a degree less than some predefined constant, $\delta n$.
Because the layers are constructed as pairs of bipartite graphs we can be sure that every vertex of every layer has a \textit{minimum} of $\delta(n)$ in-- and out--edges.
This is a powerful property to be able to both guarantee and systematically modulate at will, that X--Nets cannot offer. 
X--Nets cannot offer this because there is only the guarantee that from layer $n$ to $n+1$, every node will have $D$ edges, however the only guarantee made about minimum vertex degree from $n+1$ to $n$ is the minimum density guarantee in \cref{eq:expander_density_guarantee}.
This means that no guarantee can be made about the specific connectivity of a particular node in a particular layer with X--Nets, however $(\epsilon,\delta)$--nets \textit{do} offer a connectivity guarantee.

\section{Deterministic Construction of Super--Regular Networks}
\label{sec:detcon}

    In order to deterministically construct super--regular networks, we create layers of $(\epsilon,\delta)$--balanced matrices. To do this, we define an algebra using matrices. Full multiples of the base of the matrix system are denoted $\mathbb{A}_n, n \ge 0$; similar to the ones, tens, and hundreds positions in the decimal number system. Partial multiples of the system are denoted $\mathbb{A}_{n,s}$, where the set $s = \{1, 2, 3, 4\}$ and each digit respectively specifies the first, second, third, and fourth \textit{full diagonals} in the associated matrix. A \textit{full diagonal} is the same length as the main diagonal of a matrix ($m$, for an $m \times m$ square matrix), but may not start at $(0,0)$, and may have one or more breaks. However, it must always assign exactly one edge to each pair of vertices. Further, $\mathbb{A}_n = \mathbb{A}_{n,s=[1..4]}$.  
The symbol $\mathcal{A}$ denotes a matrix composed of one or more $\mathbb{A}_{n,s}$ matrices. 


%
%
%
%
\definecolor{pixel0}{HTML}{FFFFFF}
\definecolor{pixel1}{HTML}{000000}
\definecolor{pixel2}{HTML}{0000FF}
\definecolor{pixel3}{HTML}{800080}
\definecolor{newpurple}{HTML}{800080}
\definecolor{pixel4}{HTML}{008000}
\definecolor{newgreen}{HTML}{008000}
\definecolor{pixel5}{HTML}{FF9300}
\definecolor{neworange}{HTML}{FF9300}

\begin{center}
\begin{figure}
\begin{center}
\def\pixelsBoxTwoOneTwo{
  {1,0,0,2},
  {2,1,0,0},
  {0,2,1,0},
  {0,0,2,1}%
}
\begin{tikzpicture}
\node at (-1/2,-1/2) {$\mathbb{A}_{2_{\{1,2\}}}$:};
  \foreach \line [count=\y] in \pixelsBoxTwoOneTwo {
    \foreach \pix [count=\x] in \line {
        \draw[fill=pixel\pix] (\x/3,-\y/3) rectangle +(1/3,1/3);
    }
  }
\end{tikzpicture}
\def\pixelsBoxTwoOneThree{
  {1,0,3,0},
  {0,1,0,3},
  {3,0,1,0},
  {0,3,0,1}%
}
\hspace{.5in}
\begin{tikzpicture}
    \node at (-1/2,-1/2) {$\mathbb{A}_{2_{\{1,3\}}}$:};
  \foreach \line [count=\y] in \pixelsBoxTwoOneThree {
    \foreach \pix [count=\x] in \line {
        \draw[fill=pixel\pix] (\x/3,-\y/3) rectangle +(1/3,1/3);
    }
  }
\end{tikzpicture}
\end{center}
\caption{Significance of edge addition pattern.}
\label{fig:additionPattern}
\end{figure}
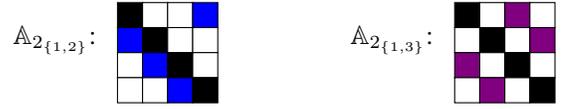
\end{center}

\begin{center}
\begin{figure}
\begin{center}
\def\pixelsBoxedZero{
   {1}%
}
\begin{tikzpicture}
  \node at (-1/8,0) {$\mathbb{A}_0$:};
  \foreach \line [count=\y] in \pixelsBoxedZero {
    \foreach \pix [count=\x] in \line {
        \draw[fill=pixel\pix, draw=white] (\x/5,-\y/5) rectangle +(1/5,1/5);
    }
  }
\end{tikzpicture}
\def\pixelsBoxedOne{
   {1,2}, 
   {2,1}
}
\begin{tikzpicture}
  \node at (-1/8,0) {$\mathbb{A}_1$:};
  \foreach \line [count=\y] in \pixelsBoxedOne {
    \foreach \pix [count=\x] in \line {
        \draw[fill=pixel\pix, draw=white] (\x/5,-\y/5) rectangle +(1/5,1/5);
    }
  }
\end{tikzpicture}
\def \pixelsBoxedTwo{
  {1,4,3,2},
  {2,1,4,3},
  {3,2,1,4},
  {4,3,2,1}%
}
\begin{tikzpicture}
  \node at (-1/8,0) {$\mathbb{A}_2$:};
  \foreach \line [count=\y] in \pixelsBoxedTwo {
    \foreach \pix [count=\x] in \line {
        \draw[fill=pixel\pix, draw=white] (\x/5,-\y/5) rectangle +(1/5,1/5);
    }
  }
\end{tikzpicture}
\def \pixelsBoxedThree{
    {1,0,4,0,3,0,2,0},
    {0,1,0,4,0,3,0,2},
    {2,0,1,0,4,0,3,0},
    {0,2,0,1,0,4,0,3},
    {3,0,2,0,1,0,4,0},
    {0,3,0,2,0,1,0,4},
    {4,0,3,0,2,0,1,0},
    {0,4,0,3,0,2,0,1}%
}
\begin{tikzpicture}
  \node at (-1/8,0) {$\mathbb{A}_3$:};
  \foreach \line [count=\y] in \pixelsBoxedThree {
    \foreach \pix [count=\x] in \line {
        \draw[fill=pixel\pix] (\x/5,-\y/5) rectangle +(1/5,1/5);
    }
  }
\end{tikzpicture}
\def \pixelsBoxedFour{
    {1,0,0,0,4,0,0,0,3,0,0,0,2,0,0,0},
    {0,1,0,0,0,4,0,0,0,3,0,0,0,2,0,0},
    {0,0,1,0,0,0,4,0,0,0,3,0,0,0,2,0},
    {0,0,0,1,0,0,0,4,0,0,0,3,0,0,0,2},
    {2,0,0,0,1,0,0,0,4,0,0,0,3,0,0,0},
    {0,2,0,0,0,1,0,0,0,4,0,0,0,3,0,0},
    {0,0,2,0,0,0,1,0,0,0,4,0,0,0,3,0},
    {0,0,0,2,0,0,0,1,0,0,0,4,0,0,0,3},
    {3,0,0,0,2,0,0,0,1,0,0,0,4,0,0,0},
    {0,3,0,0,0,2,0,0,0,1,0,0,0,4,0,0},
    {0,0,3,0,0,0,2,0,0,0,1,0,0,0,4,0},
    {0,0,0,3,0,0,0,2,0,0,0,1,0,0,0,4},
    {4,0,0,0,3,0,0,0,2,0,0,0,1,0,0,0},
    {0,4,0,0,0,3,0,0,0,2,0,0,0,1,0,0},
    {0,0,4,0,0,0,3,0,0,0,2,0,0,0,1,0},
    {0,0,0,4,0,0,0,3,0,0,0,2,0,0,0,1}%
}
\begin{tikzpicture}
  \node at (-1/8,0) {$\mathbb{A}_4$:};
  \foreach \line [count=\y] in \pixelsBoxedFour {
    \foreach \pix [count=\x] in \line {
        \draw[fill=pixel\pix] (\x/5,-\y/5) rectangle +(1/5,1/5);
    }
  }
\end{tikzpicture}
\end{center}
\caption{The construction of the first four base matrices. Edges are added in sets of $m$ during the initial construction (this can be modulated in the final stages to achieve a desired target density, $d_t$). The order of edge construction is: \textbf{first (black)}, \textbf{\color{blue}second (blue)}, \textbf{\color{newpurple}third (purple)}, \textbf{\color{newgreen}fourth (green)}. A white box means there is no edge present.}
\label{firstFive}
\end{figure}
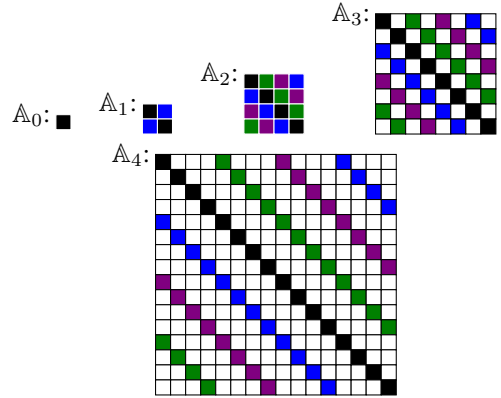
\end{center}

The fundamental unit of the matrix system is the $1 \times 1$ matrix defined by $\mathbb{A}_0$. The first compositional unit is $\mathbb{A}_1$, and may be constructed using between one and four $\mathbb{A}_0$ units. The second compositional unit is $\mathbb{A}_2$, and so on. $\mathbb{A}_2$ may be composed of between one and four $\mathbb{A}_1$ units, or between one and eight $\mathbb{A}_0$ units. In fact, every $\mathbb{A}_k$ matrix is composed of some power of two multiple of $\mathbb{A}_0$. The pattern that edges are added to subgraphs is an essential aspect to this system, and allows maintaining an ($\epsilon,\delta$)--balanced matrix. Exactly $m$ edges are added at a time, with the following ordering for $\mathbb{A}_k, k \ge 2$ ($\mathbb{A}_0$ and $\mathbb{A}_1$ follow a similar but shorter approach). \Cref{eq:AgenFunc} shows the pattern used to generate a particular size matrix:

\begin{equation}
A_{n,s} = 
\begin{cases}
	\begin{aligned}
      		\mathbb{I}_0 		&	& 	\text{if}	&	  & 1 \in s &		& \text{else} &		& 0_{n,n}  &	 & +\\
           	\mathbb{I}_{n/4} 	&	& 	\text{if}	&	  & 2 \in s &		& \text{else} &		& 0_{n,n}  &	 & +\\
                	\mathbb{I}_{n/2} 	&	& 	\text{if}	&	  & 3 \in s &		& \text{else} &		& 0_{n,n}  &	 & +\\
            	\mathbb{I}_{3n/4} 	&	& 	\text{if}	&	  & 4 \in s &		& \text{else} &		& 0_{n,n}  &	 &  \\
	\end{aligned}
   \end{cases},
   \label{eq:AgenFunc}
\end{equation}

where the $\mathbb{I}$ subscript indicates the starting row in the first column to begin the diagonal, which wraps around at the top of the matrix if necessary to become a full diagonal.

\Cref{fig:additionPattern} shows why this pattern is necessary. Essentially, it minimizes the probability of selecting an empty set for a given set size. For example, filling only the first two diagonals of $\mathbb{A}_2$ produces a bipartite graph with $d=.5$, and guarantees at least one edge in a selected subset of size 2 on each side.
However, a matrix of the same size, but with the first and third diagonals filled instead, permits the selection of an empty set (if the odd rows and even columns are selected, or vice versa).
Both matrices examined in this case are balanced, though the former has a tighter $\epsilon$--balance than the latter.
This property scales with the size of the submatrix, assuming the sizes of the selected sets scale as well. By building matrices according to the aforementioned diagonal ordering, we can deterministically create ($\epsilon$--$\delta$)--balanced matrices. \Cref{firstFive} shows the construction of the first five submatrices, along with the pattern used to add edges.

\subsection{Addition}

Any $\mathbb{A}_{(k-q)}, \forall 0 \le q \le k$ may be \textit{added} to $\mathbb{A}_k$, from 0 to $4^{(k-q)}$ times.
Each $\mathbb{A}_k$ implicilty describes how any $\mathbb{A}_{(k-q)}$ may be added to it.
Adding two matrices requres a bijection, $\beta: \mathbb{A}_r \mapsto \mathbb{A}_k$, where $r = (k - q)$, such that each index in $\mathbb{A}_r$ has exactly one distinct corresponding submatrix of size $2^{(k-q)}$ in $\mathbb{A}_k$, for all submatrices of size $2^{(k-q)}$ in $\mathbb{A}_k$ that contain a ``true'' diagonal.
Then, compute $\mathcal{A} = \beta(\mathbb{A}_q) \cup \mathbb{A}_k$, which copies $\mathbb{A}_q$ into every previously non--empty submatrix.  

For example, to add $\mathbb{A}_{4_{1,2}} + \mathbb{A}_{3_{1,2}},$ define $\beta: \mathbb{A}_1 \mapsto \mathbb{A}_{4_{1,2}}$, and $\mathcal{A} = \beta(\mathbb{A}_{3_{1,2}}) \cup \mathbb{A}_{4_{1,2}}$. This operation is illustrated in \cref{simpleAddition}:

\begin{center}
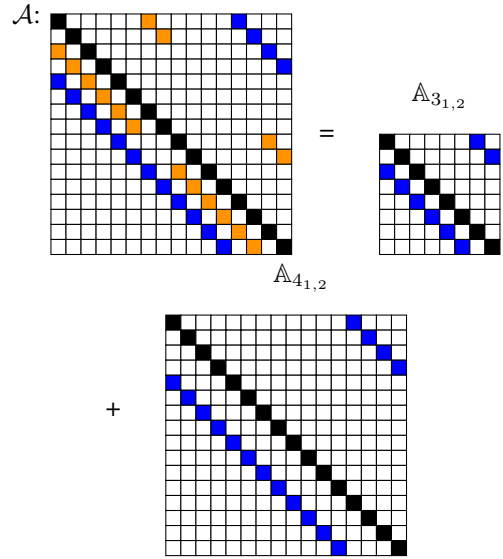
\begin{figure}
\begin{center}
\def \pixelsBoxedFourSimpleAddition{
    {1,0,0,0,0,0,5,0,0,0,0,0,2,0,0,0},
    {0,1,0,0,0,0,0,5,0,0,0,0,0,2,0,0},
    {5,0,1,0,0,0,0,0,0,0,0,0,0,0,2,0},
    {0,5,0,1,0,0,0,0,0,0,0,0,0,0,0,2},
    {2,0,5,0,1,0,0,0,0,0,0,0,0,0,0,0},
    {0,2,0,5,0,1,0,0,0,0,0,0,0,0,0,0},
    {0,0,2,0,5,0,1,0,0,0,0,0,0,0,0,0},
    {0,0,0,2,0,5,0,1,0,0,0,0,0,0,0,0},
    {0,0,0,0,2,0,0,0,1,0,0,0,0,0,5,0},
    {0,0,0,0,0,2,0,0,0,1,0,0,0,0,0,5},
    {0,0,0,0,0,0,2,0,5,0,1,0,0,0,0,0},
    {0,0,0,0,0,0,0,2,0,5,0,1,0,0,0,0},
    {0,0,0,0,0,0,0,0,2,0,5,0,1,0,0,0},
    {0,0,0,0,0,0,0,0,0,2,0,5,0,1,0,0},
    {0,0,0,0,0,0,0,0,0,0,2,0,5,0,1,0},
    {0,0,0,0,0,0,0,0,0,0,0,2,0,5,0,1}%
}
\begin{tikzpicture}
    \node at (-1/8,0) {$\mathcal{A}$:};
  \foreach \line [count=\y] in \pixelsBoxedFourSimpleAddition {
    \foreach \pix [count=\x] in \line {
        \draw[fill=pixel\pix] (\x/5,-\y/5) rectangle +(1/5,1/5);
    }
  }
\end{tikzpicture}
\def \pixelsBoxedThreeOneTwo{
    {1,0,0,0,0,0,2,0},
    {0,1,0,0,0,0,0,2},
    {2,0,1,0,0,0,0,0},
    {0,2,0,1,0,0,0,0},
    {0,0,2,0,1,0,0,0},
    {0,0,0,2,0,1,0,0},
    {0,0,0,0,2,0,1,0},
    {0,0,0,0,0,2,0,1}%
}
\begin{tikzpicture}
    \node at (1,.5) {$\mathbb{A}_{3_{1,2}}$};
    \node at (-1/2,0) {=};
  \foreach \line [count=\y] in \pixelsBoxedThreeOneTwo {
    \foreach \pix [count=\x] in \line {
        \draw[fill=pixel\pix] (\x/5,-\y/5) rectangle +(1/5,1/5);
    }
  }
\end{tikzpicture}
\def \pixelsBoxedFourOneTwo{
    {1,0,0,0,0,0,0,0,0,0,0,0,2,0,0,0},
    {0,1,0,0,0,0,0,0,0,0,0,0,0,2,0,0},
    {0,0,1,0,0,0,0,0,0,0,0,0,0,0,2,0},
    {0,0,0,1,0,0,0,0,0,0,0,0,0,0,0,2},
    {2,0,0,0,1,0,0,0,0,0,0,0,0,0,0,0},
    {0,2,0,0,0,1,0,0,0,0,0,0,0,0,0,0},
    {0,0,2,0,0,0,1,0,0,0,0,0,0,0,0,0},
    {0,0,0,2,0,0,0,1,0,0,0,0,0,0,0,0},
    {0,0,0,0,2,0,0,0,1,0,0,0,0,0,0,0},
    {0,0,0,0,0,2,0,0,0,1,0,0,0,0,0,0},
    {0,0,0,0,0,0,2,0,0,0,1,0,0,0,0,0},
    {0,0,0,0,0,0,0,2,0,0,0,1,0,0,0,0},
    {0,0,0,0,0,0,0,0,2,0,0,0,1,0,0,0},
    {0,0,0,0,0,0,0,0,0,2,0,0,0,1,0,0},
    {0,0,0,0,0,0,0,0,0,0,2,0,0,0,1,0},
    {0,0,0,0,0,0,0,0,0,0,0,2,0,0,0,1}%
}
\begin{tikzpicture}
    \node at (2,.5) {$\mathbb{A}_{4_{1,2}}$};
    \node at (-1/2,-1.25) {+};
  \foreach \line [count=\y] in \pixelsBoxedFourOneTwo {
    \foreach \pix [count=\x] in \line {
        \draw[fill=pixel\pix] (\x/5,-\y/5) rectangle +(1/5,1/5);
    }
  }
\end{tikzpicture}
\end{center}
\caption{Simple addition via a submatrix bijection.}
\label{simpleAddition}
\end{figure}
\end{center}

Essentially, we create a grid of size $|\mathbb{A}_q|$, and add copies of $\mathbb{A}_r$ in the locations that $\mathbb{A}_k$ has a ``true'' diagonal already in each respective submatrix. Submatrices, similarly to individual edges, are added to the primary matrix according to the edge ordering constraint (shown in \cref{firstFive}). 
In fact, $\mathbb{A}_0$ is both a submatrix and an edge, and could be added to any $\mathbb{A}_k, \forall k$ in exactly the manner described above. If $\mathbb{A}_k$ has 4 (denoted $\mathbb{A}_{k_4}$) \textit{full diagonals} (diagonals of total length $m$) worth of edges, adding $\mathbb{A}_{{(k-1)}_4}$ to it four times will increase the density by a factor of $4m$. This is because two of the diagonals that $\mathbb{A}_{{(k-1)}_4}$ would add, have already been added by the completion of $\mathbb{A}_{k_4}$. As a result, we have the ability to modulate $A$'s density by $\pm \frac{m}{m^2}$.

Using this process, we are able to deterministically build a number of bipartite ($\epsilon,\delta$)--super--regular pairs.
However, in the context of multiple stacked layers of super--regular pairs, this edge assignment process structurally limits information mixing from one layer to the next.
In order to mitigate this, we randomly permute the node ordering after edge assignment. 
This ensures each ($\epsilon,\delta$)--super--regular pair continues to satisfy all required properties, while also ensuring uniform mixing as discussed by \cite{prabhu2017deep}.

\section{Experiments}
\label{sec:experiments}

In order to demonstrate that our construction of ($\epsilon,\delta$)--super--regular pairs, and thereby SRNs is practical, we compared the performance of SRNs and X--Nets using four different CNN architectures.
We modified the codebase used by \cite{prabhu2017deep} (implemented with PyTorch), and ran experiments using the CIFAR-10 dataset.
The architectures tested were VGG--16 with batch normalization, DenseNet ($k=8$), MobileNet, and ResNet50. 
The original codebase used sparse bipartite linear and 2D convolutional expander graph layers with randomly assigned edges in place of many fully--connected layers (but not all).
Our modifications replaced all bipartite expander graph layers with $(\epsilon,\delta)$--super--regular layers of the same density, with deterministically assigned edges whose nodes where then randomly permuted in order to guarantee uniform mixing. 
Each architecture was tested ten different times using $(\epsilon,\delta)$--super--regular layers and expander layers.
We used a batch size of 128 with stochastic gradient descent.

\section{Results}
\label{sec:results}

\Cref{table:diff} shows the average best top-1 precision scores over all ten trials after 100 epochs for each architecture we tested.
The rightmost column shows the absolute value of the difference between the X--Net and SRN implementations.
In terms of the averages, the X--Net implementation slightly outperformed the SRN. 

\Cref{fig:loss,fig:prec} show the training loss and validation set top--1 precision scores for each of the 10 trials of each network type for both the X--Net and SRN implementations, for each of the 100 epochs. 
The green lines display SRN trials, while the black lines show the X--Net trials.

\setlength{\tabcolsep}{4pt}
\begin{table}
\begin{center}
\caption{Average best top--1 precision scores for each architecture after training 100 epochs on the CIFAR--10 (batch size = 128)}
\label{table:diff}
\begin{tabular}{lllc}
\hline\noalign{\smallskip}
Architecture & SRN & X--Net & |SRN - X--Net| \\
\noalign{\smallskip}
\hline
\noalign{\smallskip}
VGG-16 BN  & 85.66\% & 86.17\% & 0.51\%  \\
MobileNet & 74.46\% & 74.87\% & 0.41\%  \\
DenseNet & 75.43\% & 75.74\% & 0.31\% \\
ResNet50 & 76.22 \%  & 76.78 \% & 0.56\%  \\
\hline
\end{tabular}
\end{center}
\end{table}
\setlength{\tabcolsep}{1.4pt}

\vspace{-.5in}

\begin{figure}
\centering
\subfloat[]{
	\includegraphics[width=.9\linewidth]{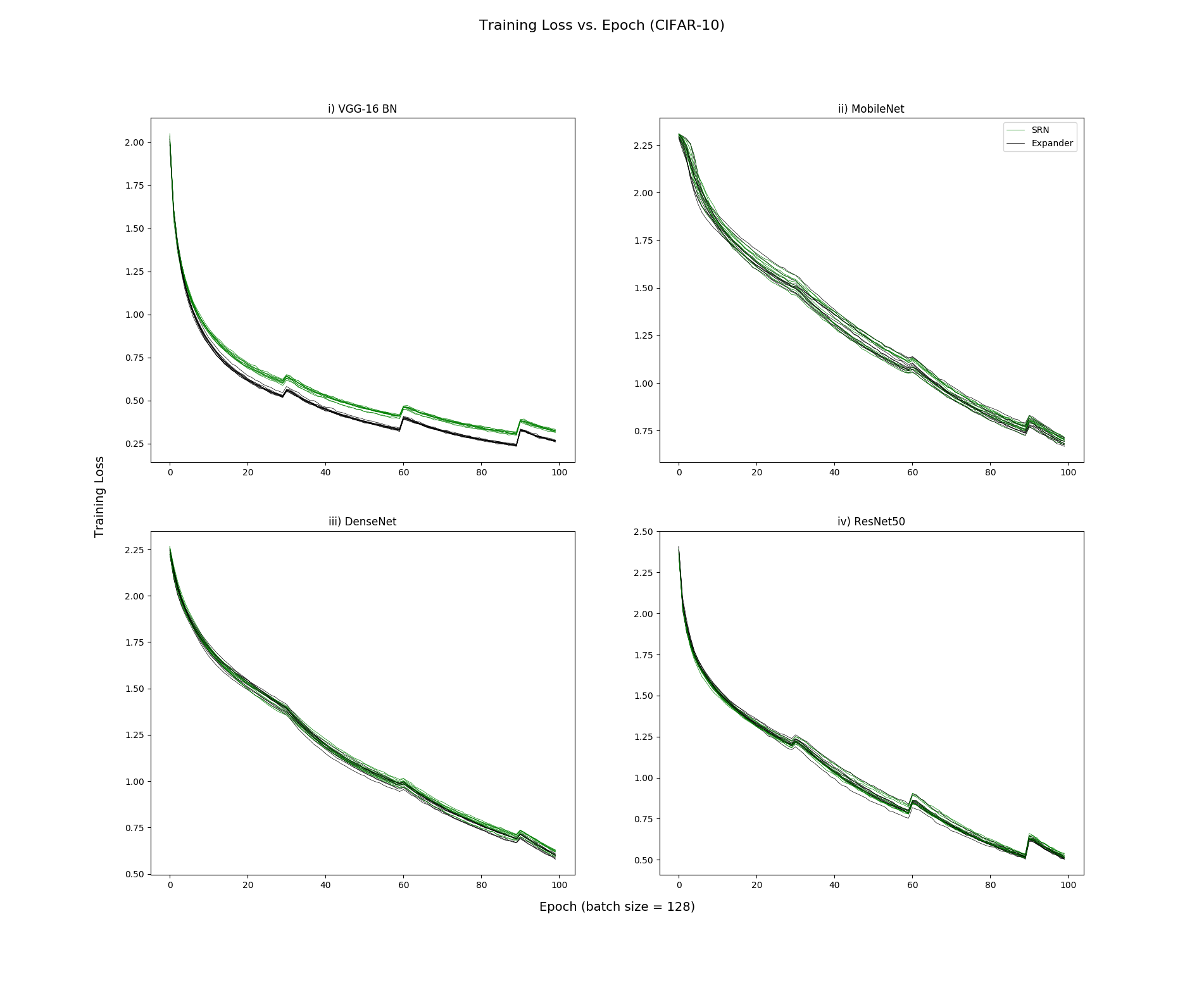}
	\label{fig:loss}
}

\subfloat[]{
	\includegraphics[width=.9\linewidth]{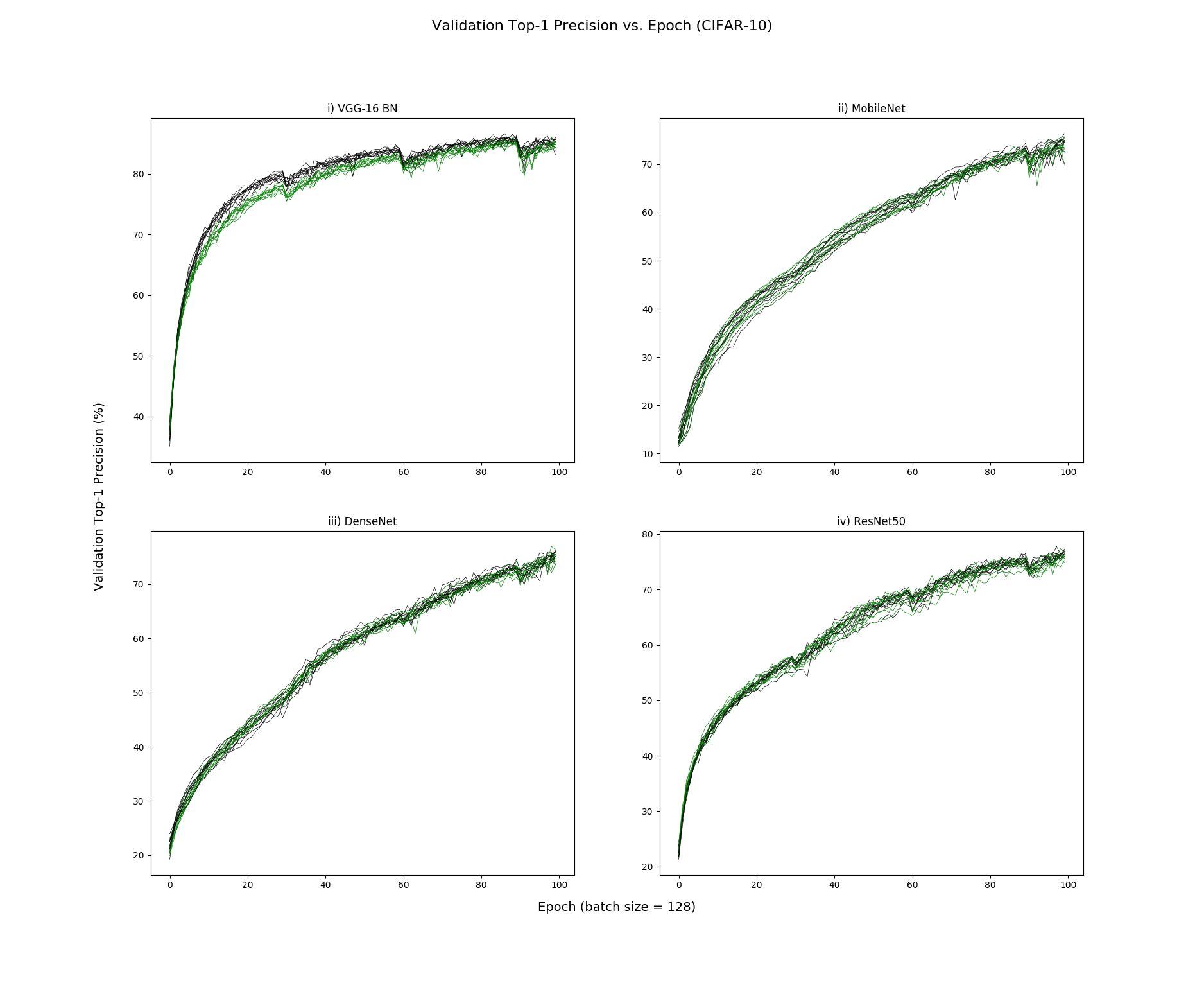}
	\label{fig:prec}
}
\caption{(a) Training loss vs. epoch for the four architectures tested. For each architecture, the green lines are each of the 10 trials using the SRN, while the black lines are each of the 10 trials using the X--Net. (b) Validation top--1 precision vs. epoch for the four architectures tested, showing 10 trials for each network. Architectures are (i) VGG-16 with batch normalization; (ii) MobileNet; (iii) DenseNet; and (iv) ResNet50}
\end{figure}

\vspace{.5in}

\section{Discussion}
\label{sec:discussion}

We expected to see SRNs outperform X--Nets specifically \textit{because} of the subgraph uniformity guaranteed by $(\epsilon,\delta)$--super--regular pairs. 
However, we found that SRNs consistently performed comparably (nearly as well as) X--Nets, as shown by \cref{fig:loss,fig:prec}. 
The VGG architecture is an interesting outlier, where the X--Net seemed to consistently outperform the SRN. 

As discussed in \cref{ssec:advantages}, SRNs put a lower bound on X--Net subgraph sparsity. 
A consequence of this is that the potential to have an uneven distribution of in--edges to a given node decreases.
This may be a primary driver of the slight observed decrease in overall performance of SRNs as compared to X--Nets.
Essentially, due to non--uniformities in the data, SRNs lose the advantage that X--Nets have of being able to relatively over-- and under--utilize certain nodes or paths. 
Minimal differences in the average best top--1 precision over the 100 training epochs across all architectures tested support this conclusion, as shown in \cref{table:diff}, as do the tightly clustered and interwoven loss and precision plots in \cref{fig:loss,fig:prec}.
In all cases, the densities of the ($\epsilon,\delta$)--super--regular layers that replaced the expander layers were identical; so it is likely that the discrepancy is a product of the differences in edge distribution.

However, there are two primary algorithmic advantages of SRNs over X--Nets.
First, $(\epsilon,\delta)$--super--regular pairs guarantee the ability to embed predetermined paths into a network, as a result of the connectivity.
Second, because the edges of SRNs are deterministically constructed, SRNs have the potential to be carefully augmented or tuned. 
As deep learning research continues to evolve, this characteristic has the potential to become increasingly important.

\section{Conclusions and future work}
\label{sec:conclusions}

We introduced $\epsilon$-- and ($\epsilon,\delta$)--balanced matrices, established a relationship  between $\epsilon$--balanced matrices and $\epsilon$--regular pairs, and between ($\epsilon,\delta$)--balanced matrices and ($\epsilon,\delta$)--super--regular pairs. We presented a method to construct pseudo--deterministic SRNs, and showed that SRNs produce comparable results to X--Nets. Further, we discussed the advantages that SRNs have over X--Nets. Specifically, SRNs promise greater network connectivity and uniformity. This means SRNs are inherently more tunable than X--Nets. Furthermore, due to their comparable performance to X--Nets and additional properties, our results suggest that SRNs are suitable replacements for FCNs. Further work is necessary to verify this. 

Future work requires us to understand why, despite identical layer densities, X--Nets seemed to slightly outperform SRNs in terms of training loss and validation precision.
Another direction to take future work is to extend the notion of transfer learning to embed multiple unrelated pre-trained sparse networks into a slightly larger SRN, and begin training for a more complex task using the newly embedded SRN as a starting point. 
Additionally, it would be helpful to experimentally determine the impact that increasing sparsity has on the performance differential between SRNs and X--Nets.
Finally, a fully deterministic construction of ($\epsilon,\delta$)--super--regular pairs may facilitate engineering specific network architectures. One possible way to approach this is via a deterministic node permutation.

\bibliographystyle{IEEEtran}
\bibliography{super_regular_nets}

\begin{thebibliography}{10}
\providecommand{\url}[1]{#1}
\csname url@samestyle\endcsname
\providecommand{\newblock}{\relax}
\providecommand{\bibinfo}[2]{#2}
\providecommand{\BIBentrySTDinterwordspacing}{\spaceskip=0pt\relax}
\providecommand{\BIBentryALTinterwordstretchfactor}{4}
\providecommand{\BIBentryALTinterwordspacing}{\spaceskip=\fontdimen2\font plus
\BIBentryALTinterwordstretchfactor\fontdimen3\font minus
  \fontdimen4\font\relax}
\providecommand{\BIBforeignlanguage}[2]{{%
\expandafter\ifx\csname l@#1\endcsname\relax
\typeout{** WARNING: IEEEtran.bst: No hyphenation pattern has been}%
\typeout{** loaded for the language `#1'. Using the pattern for}%
\typeout{** the default language instead.}%
\else
\language=\csname l@#1\endcsname
\fi
#2}}
\providecommand{\BIBdecl}{\relax}
\BIBdecl

\bibitem{thom2016}
\BIBentryALTinterwordspacing
M.~Thom and G.~Palm, ``Sparse activity and sparse connectivity in supervised
  learning,'' \emph{CoRR}, vol. abs/1603.08367, 2016. [Online]. Available:
  \url{http://arxiv.org/abs/1603.08367}
\BIBentrySTDinterwordspacing

\bibitem{esteva2019guide}
A.~Esteva, A.~Robicquet, B.~Ramsundar, V.~Kuleshov, M.~DePristo, K.~Chou,
  C.~Cui, G.~Corrado, S.~Thrun, and J.~Dean, ``A guide to deep learning in
  healthcare,'' \emph{Nature medicine}, vol.~25, no.~1, p.~24, 2019.

\bibitem{mayr2016deeptox}
A.~Mayr, G.~Klambauer, T.~Unterthiner, and S.~Hochreiter, ``Deeptox: toxicity
  prediction using deep learning,'' \emph{Frontiers in Environmental Science},
  vol.~3, p.~80, 2016.

\bibitem{lecun2015deep}
Y.~LeCun, Y.~Bengio, and G.~Hinton, ``Deep learning,'' \emph{nature}, vol. 521,
  no. 7553, p. 436, 2015.

\bibitem{srivastava14a}
\BIBentryALTinterwordspacing
N.~Srivastava, G.~Hinton, A.~Krizhevsky, I.~Sutskever, and R.~Salakhutdinov,
  ``Dropout: A simple way to prevent neural networks from overfitting,''
  \emph{Journal of Machine Learning Research}, vol.~15, pp. 1929--1958, 2014.
  [Online]. Available: \url{http://jmlr.org/papers/v15/srivastava14a.html}
\BIBentrySTDinterwordspacing

\bibitem{prabhu2017deep}
A.~Prabhu, G.~Varma, and A.~Namboodiri, ``Deep expander networks: Efficient
  deep networks from graph theory,'' \emph{arXiv preprint arXiv:1711.08757},
  2017.

\bibitem{komlos1997blow}
J.~Koml{\'o}s, G.~N. S{\'a}rk{\"o}zy, and E.~Szemer{\'e}di, ``Blow-up lemma,''
  \emph{Combinatorica}, vol.~17, no.~1, pp. 109--123, 1997.

\bibitem{molchanov2017variational}
D.~Molchanov, A.~Ashukha, and D.~Vetrov, ``Variational dropout sparsifies deep
  neural networks,'' in \emph{Proceedings of the 34th International Conference
  on Machine Learning-Volume 70}.\hskip 1em plus 0.5em minus 0.4em\relax JMLR.
  org, 2017, pp. 2498--2507.

\bibitem{guoSparseAdversarial}
\BIBentryALTinterwordspacing
Y.~Guo, C.~Zhang, C.~Zhang, and Y.~Chen, ``Sparse dnns with improved
  adversarial robustness,'' in \emph{Advances in Neural Information Processing
  Systems 31}, S.~Bengio, H.~Wallach, H.~Larochelle, K.~Grauman,
  N.~Cesa-Bianchi, and R.~Garnett, Eds.\hskip 1em plus 0.5em minus 0.4em\relax
  Curran Associates, Inc., 2018, pp. 242--251. [Online]. Available:
  \url{http://papers.nips.cc/paper/7308-sparse-dnns-with-improved-adversarial-robustness.pdf}
\BIBentrySTDinterwordspacing

\bibitem{moosavi2016deepfool}
S.-M. Moosavi-Dezfooli, A.~Fawzi, and P.~Frossard, ``Deepfool: a simple and
  accurate method to fool deep neural networks,'' in \emph{Proceedings of the
  IEEE conference on computer vision and pattern recognition}, 2016, pp.
  2574--2582.

\bibitem{wen2016lss}
\BIBentryALTinterwordspacing
W.~Wen, C.~Wu, Y.~Wang, Y.~Chen, and H.~Li, ``Learning structured sparsity in
  deep neural networks,'' in \emph{Advances in Neural Information Processing
  Systems 29}, D.~D. Lee, M.~Sugiyama, U.~V. Luxburg, I.~Guyon, and R.~Garnett,
  Eds.\hskip 1em plus 0.5em minus 0.4em\relax Curran Associates, Inc., 2016,
  pp. 2074--2082. [Online]. Available:
  \url{http://papers.nips.cc/paper/6504-learning-structured-sparsity-in-deep-neural-networks.pdf}
\BIBentrySTDinterwordspacing

\bibitem{tartaglione2018}
\BIBentryALTinterwordspacing
E.~Tartaglione, S.~Leps\o~y, A.~Fiandrotti, and G.~Francini, ``Learning sparse
  neural networks via sensitivity-driven regularization,'' in \emph{Advances in
  Neural Information Processing Systems 31}, S.~Bengio, H.~Wallach,
  H.~Larochelle, K.~Grauman, N.~Cesa-Bianchi, and R.~Garnett, Eds.\hskip 1em
  plus 0.5em minus 0.4em\relax Curran Associates, Inc., 2018, pp. 3878--3888.
  [Online]. Available:
  \url{http://papers.nips.cc/paper/7644-learning-sparse-neural-networks-via-sensitivity-driven-regularization.pdf}
\BIBentrySTDinterwordspacing

\bibitem{zhu2018improving}
X.~Zhu, W.~Zhou, and H.~Li, ``Improving deep neural network sparsity through
  decorrelation regularization.'' in \emph{IJCAI}, 2018, pp. 3264--3270.

\bibitem{sun2016cvpr}
Y.~Sun, X.~Wang, and X.~Tang, ``Sparsifying neural network connections for face
  recognition,'' in \emph{The IEEE Conference on Computer Vision and Pattern
  Recognition (CVPR)}, June 2016.

\bibitem{komlos2000regularity}
J.~Koml{\'o}s, A.~Shokoufandeh, M.~Simonovits, and E.~Szemer{\'e}di, ``The
  regularity lemma and its applications in graph theory,'' in \emph{Summer
  School on Theoretical Aspects of Computer Science}.\hskip 1em plus 0.5em
  minus 0.4em\relax Springer, 2000, pp. 84--112.

\bibitem{kalantari1997}
\BIBentryALTinterwordspacing
B.~Kalantari, L.~Khachiyan, and A.~Shokoufandeh, ``On the complexity of matrix
  balancing,'' \emph{SIAM J. Matrix Anal. Appl.}, vol.~18, no.~2, pp. 450--463,
  Apr. 1997. [Online]. Available:
  \url{http://dx.doi.org/10.1137/S0895479895289765}
\BIBentrySTDinterwordspacing

\end{thebibliography}

\end{document}